\documentclass[sigconf,9pt,screen=true,anonymous=false,bookmarks=false]{acmart}
\usepackage{lipsum}
\usepackage{graphicx}
\usepackage{footnote}

\usepackage{amsmath}
\usepackage{mathtools}
\usepackage{mathrsfs}     
\usepackage{comment}
\usepackage[subrefformat=parens,labelformat=parens]{subfig}
\usepackage{bm}
\usepackage{multirow}
\usepackage{threeparttable,booktabs}
\usepackage{blkarray}
\usepackage{tikz}
\usetikzlibrary{positioning,calc,fit,decorations.pathmorphing,shapes.geometric,shapes.gates.logic.US,calc}
\usepackage{balance}
\usepackage{courier}                                       
\usepackage{cleveref}                                      
\usepackage[mathcal]{eucal}
\usepackage[]{algpseudocode}                               
\algrenewcommand\textproc{\texttt}
\makeatletter\let\float@addtolists\relax\makeatother
\usepackage{algorithm}
\usepackage{filecontents}                                  
\usepackage{pgfplots}
\usepackage{pgfplotstable}
\usepgfplotslibrary{groupplots}
\pgfplotsset{compat=newest}
\usepackage[figuresright]{rotating}
\usepackage[inline]{enumitem}

\DeclareMathOperator*{\argmax}{argmax}

\theoremstyle{plain}
\newtheorem{mytheorem}{\textbf{Theorem}}

\theoremstyle{definition}

\usepackage[skip=1pt]{caption}            
\definecolor{CUHKorange}{RGB}{244,106,18} 
\definecolor{CUHKblue}{RGB}{0,111,190}    
\definecolor{CUHKgreen}{RGB}{0,127,128}   
\definecolor{CUHKred}{RGB}{228,46,36}     
\definecolor{CUHKyellow}{RGB}{198,148,34} 
\definecolor{CUHKdark}{RGB}{114,44,114}   
\definecolor{CUHKmiddle}{RGB}{144,44,144} 

\usepackage{tcolorbox}
\tcbuselibrary{skins,breakable}
    {\endtcolorbox}

\graphicspath{{./figs/}}
\setcopyright{acmcopyright}

\newif\ifrev
\revtrue
\ifrev
  \newcommand{\bei}[1]{{\color{red} [Bei: #1]}} 
\else
  \newcommand{\bei}[1]{}
\fi

\definecolor{myblue}{RGB}{29,114,221}    
\definecolor{myyellow}{RGB}{255,255,191} 
\definecolor{myorange}{RGB}{244,106,18}  
\definecolor{mygray}{RGB}{102,102,102}   
\definecolor{mypink}{RGB}{252,228,215}   

\definecolor{CUpurple}{RGB}{136,43,142}
\definecolor{CUlpurple}{RGB}{165,133,182}
\definecolor{CUgold}{RGB}{221,163,0}
\definecolor{CUribbon}{RGB}{244,223,176}
\definecolor{CUblack}{RGB}{34,24,21}
\definecolor{PKUred}{RGB}{126,24,28}
\definecolor{gray6}{gray}{0.6}
\definecolor{gray7}{gray}{0.7}
\definecolor{gray8}{gray}{0.8}
\definecolor{gray9}{gray}{0.9}

\copyrightyear{2022} 
\acmYear{2022} 
\setcopyright{acmcopyright}\acmConference[ICCAD '22]{IEEE/ACM International Conference on Computer-Aided Design}{October 30-November 3, 2022}{Gainesville, FL, USA}
\acmBooktitle{IEEE/ACM International Conference on Computer-Aided Design (ICCAD '22), October 30-November 3, 2022, Gainesville, FL, USA}
\acmPrice{15.00}
\acmDOI{10.1145/3508352.3549468}
\acmISBN{978-1-4503-9217-4/22/10}

\begin{document}
\date{}

\title{
    AdaOPC: A Self-Adaptive Mask Optimization Framework For Real Design Patterns
}


\author{Wenqian Zhao}
\affiliation{
    \institution{CUHK}
}

\author{Xufeng Yao}
\affiliation{
    \institution{CUHK}
}

\author{Ziyang Yu}
\affiliation{
    \institution{CUHK}
}

\author{Guojin Chen}
\affiliation{
    \institution{CUHK}
}

\author{Yuzhe Ma}
\affiliation{
    \institution{HKUST(GZ)}
}

\author{Bei Yu}
\affiliation{
    \institution{CUHK}
}

\author{Martin D.F.~Wong}
\affiliation{
    \institution{CUHK}
}

\begin{abstract}
    Optical proximity correction (OPC) is a widely-used resolution enhancement technique (RET) for printability optimization. 
    Recently, rigorous numerical optimization and fast machine learning are the research focus of OPC in both academia and industry, each of which complements the other in terms of robustness or efficiency. 
    We inspect the pattern distribution on a design layer and find that different sub-regions have different pattern complexity.
    Besides, we also find that many patterns repetitively appear in the design layout, and these patterns may possibly share optimized masks. 
    We exploit these properties and propose a self-adaptive OPC framework to improve efficiency. 
    Firstly we choose different OPC solvers adaptively for patterns of different complexity from an extensible solver pool to reach a speed/accuracy co-optimization. 
    Apart from that, we prove the feasibility of reusing optimized masks for repeated patterns and hence, build a graph-based dynamic pattern library reusing stored masks to further speed up the OPC flow. 
    Experimental results show that our framework achieves substantial improvement in both performance and efficiency.
\end{abstract}

\maketitle
\pagestyle{plain}

\section{Introduction}\label{intro}
\label{sec:intro}
Over the past few decades, VLSI technology node has been continuously shrinking, resulting non-neglectable lithography proximity effect, which affects the real manufacturability \cite{pan2013design}. Resolution enhancement techniques (RETs) are utilized to improve the printability in the lithography process. Optical Proximity Correction (OPC) is one of the widely used RETs to optimize mask printability by compensating for the diffraction effect in the lithography process. 

OPC approaches can be categorized into: 
\begin{enumerate*}
    \item rule-based OPC \cite{park2000efficient}, 
    \item model-based OPC \cite{kuang2015robust, su2016fast, matsunawa2015optical}, 
    \item inverse lithography technique (ILT)-based OPC \cite{poonawala2007mask, ma2020unified, yu2021gpu} and 
    \item machine learning (ML)-based OPC \cite{yang2019gan, jiang2019fast, geng2019sraf, chen2021damo}.
\end{enumerate*}
Rule-based methods solve the problem heuristically, which is simple and fast but only suitable for less aggressive designs. 
Model-based OPCs mathematically model the lithography process and move/shift the edge fractures accordingly, ensuring mask fidelity but restricted by the solutions space in advanced technology nodes. 
ILT-based methods solve the inverse problem of the imaging system through optimizing an objective function, which is the most performant analytical method to tackle the OPC problem. 
As recent years witnessed the rapid development of machine learning algorithms and hardware,
ML-based OPCs have shown remarkable speed-up in the OPC flows and are prevailing in design for manufacturing (DFM) academia. \cite{yang2019gan} and \cite{chen2021develset} use deep learning models for initial mask generation to reduce the iterations number of ILT.
\cite{jiang2020neural} uses a deep learning model to simulate the conventional ILT correction process.
A lethal drawback is that machine learning model is a data-driven black box. Such methods are not guaranteed to work for some critical patterns. 
In summary, no approach is flawless and shows absolute superiority over others. Patterns with different complexity require different approaches.


\begin{figure}[tb!]
    \raggedleft
    \includegraphics[width=\linewidth]{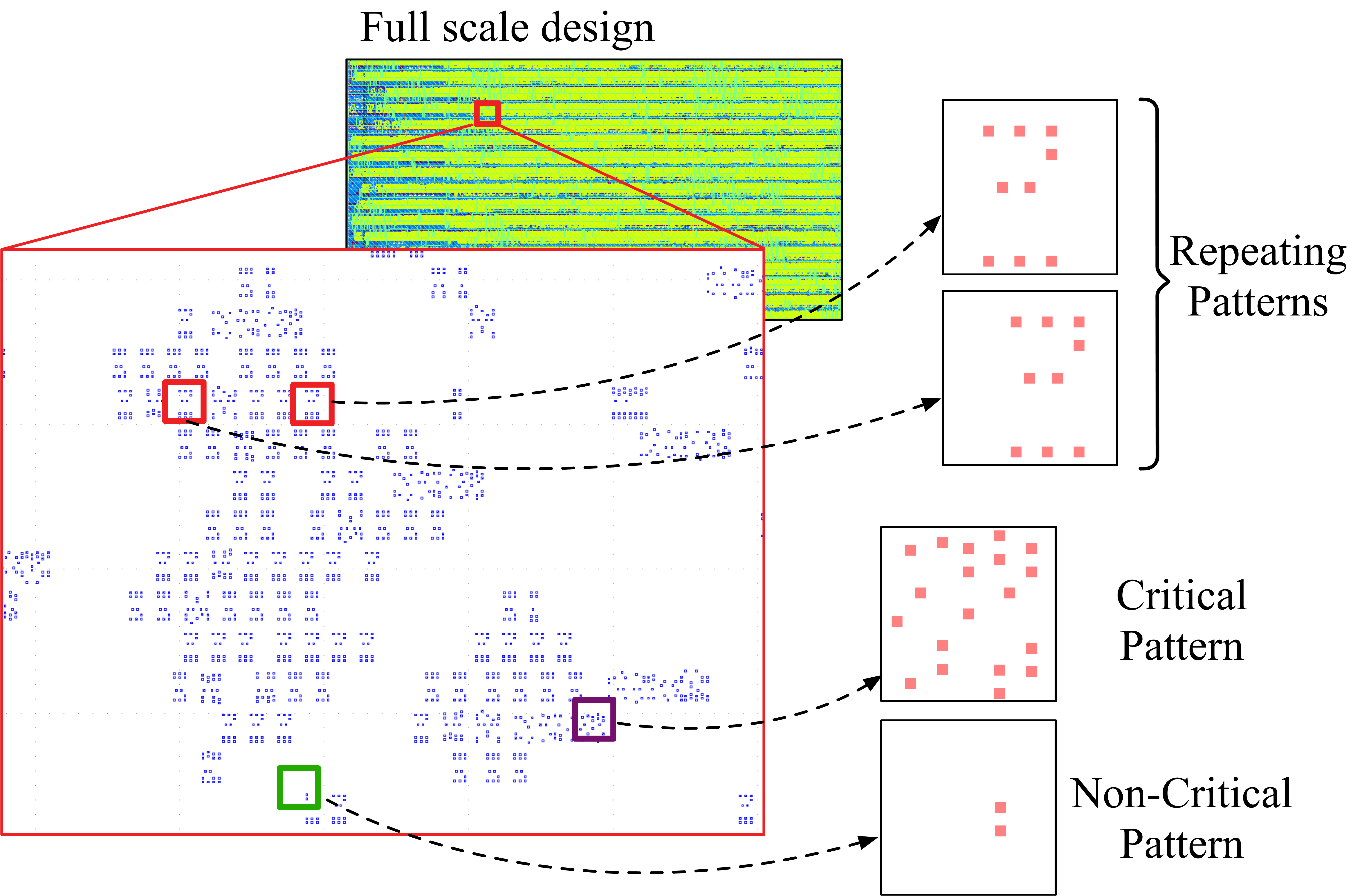} 
    \caption{Visualization of a real design layer. Two discoveries motivated our OPC framework design: 1. Patterns scattered unevenly along the design layout with different complexity. We denote complicated patterns as critical and simple patterns as non-critical. 2. Patterns have large ratio of repetition on a full layout. }
    \label{fig:whole design}
\end{figure}

Achieving desired OPC results with high efficiency on real designs require a systematic analysis on pattern distribution and complexity. 
By inspecting a real design, we come up with a few properties that can be leveraged to assist the analysis. 
Firstly, we notice that there exists a \emph{diversity} of pattern density, which indicates that some regions are dense while some regions are sparse, as visualized in \Cref{fig:whole design}, which implicates different kinds of OPC solutions are required. 
Moreover, as we take a closer observation of each sub-region, there also exists a \emph{similarity} of the pattern distribution in different sub-regions. 
Many patterns are repeatedly placed on the full design layer. These patterns have similar geometric shapes and are placed at different locations. Such pattern repetition enables us to utilize their shared geometric characteristics, which motivates the idea that the OPC solution of a pattern can be reused in a similar pattern for efficiency improvement.  
Motivated by these observations, we propose a self-adaptive framework, namely AdaOPC, for conducting OPC on real designs.

Firstly, AdaOPC is equipped with pattern analysis that can classify a sub-region as critical or non-critical, such that the proper OPC solver can be selected. 
More specifically, for densely scattered sub-regions, not only the diffraction effect but also the optical interference of incident light caused by neighboring components will jointly affect the final printed image. 
Such complex patterns for OPC are regarded as critical and more suitable for robust yet rigorous numerical optimization methods for higher manufacturability. 
In contrast, sub-regions with sparsely scattered patterns are simpler, and the mask optimization process is more suitable for machine learning models to learn with superior inference speed.  

Secondly, given that many patterns are repetitive on the design layer with the same geometric shape, as shown in \Cref{fig:whole design},  we investigate the feasibility of reusing the optimized masks of repeating patterns to avoid redundant iterations of OPC. 
There are three obstacles standing in the way of our idea:
\begin{enumerate*}
    \item As shown in \Cref{fig:slicing}, slicing the large design layout into small patterns inevitably result in a location shift of patterns with the same geometric shape. Whether and how can an optimized mask with location shift be reused?
    \item Given the query pattern, how to match a same pattern from a large number of stored ones accurately within an acceptable time? 
    \item How to measure the geometric similarity of patterns with location shift?
\end{enumerate*}
In response to the mentioned three questions, we build a \textbf{dynamic pattern library} with online updating to store and reuse repeating patterns and optimized masks by constructing a dynamic hierarchical graph.
We mathematically prove the shift equivariance property of the lithography process to show the feasibility of mask reuse by calculating the shift of design pattern and calibrating the mask. 
A graph-based approximation nearest neighbor search for pattern matching within a short query time. 
We summarize the contributions of this paper as follows:
\begin{itemize}
    \item We propose a self-adaptive workflow for flexible OPC solver selection.
    \item We prove the feasibility of mask resue to speed up the OPC process for real design patterns and provide an efficient mask shift calibration method in practice.
    \item We generate design patterns embedding by supervised contrastive learning for similarity measurement and pattern matching.
    \item We construct a dynamic pattern library using a hierarchical graph with online update along with a greedy graph-based nearest neighbor search for fast matching.
    \item With experiments on different pattern cases from a real design layout, we proved our framework can reduce over 90\% runtime while still preserving the optimal OPC performance.
\end{itemize}
\begin{figure}
    \centering
    \includegraphics[width=0.9\linewidth]{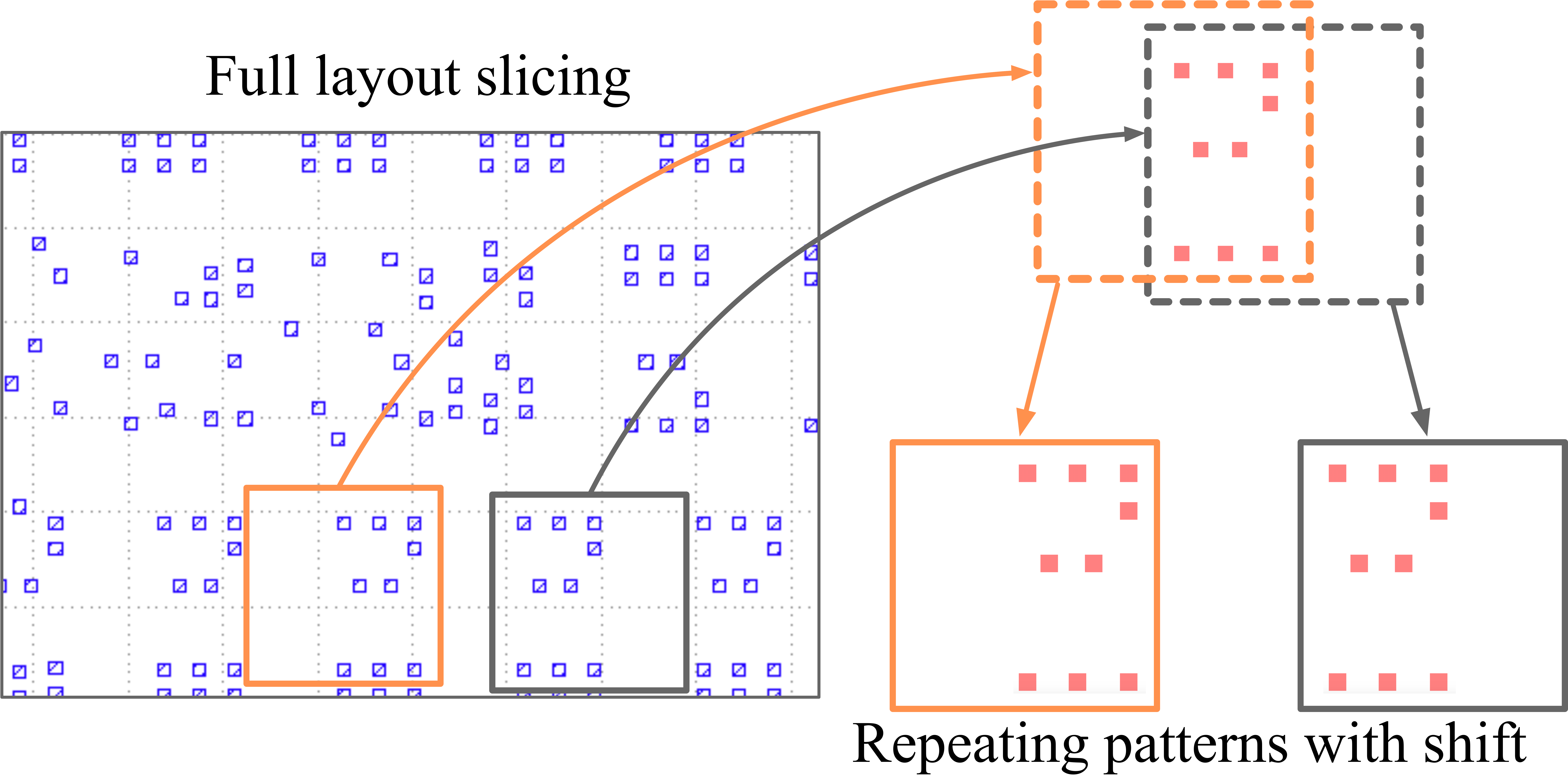}
    \caption{Slicing repeating full layout inevitably causes some location shift on repeating patterns.}
    \label{fig:slicing}
\end{figure}

\section{Preliminaries}

\subsection{Lithography Simulation Model} \label{litho_model}
During the lithography process, an input mask $\mathbf{M}$ is projected through layers of optical lens onto a wafer plane. The intensity after optical system $\mathbf{I}$, namely the aerial image, leaves a coating on the wafer with photoresist to form the resulting pattern $\mathbf{Z}$. The conventional simulation of the lithography process is composed of 2 consecutive components: optical projection model and photoresist model. 

For the projection process, Hopkins diffraction model \cite{hopkins1951concept} has been widely used to analyze coherent imaging system mathematically. To avoid the computation complexity of the Hopkins model, a singular value decomposition model (SVD)-based approximation has been proposed by \cite{cobb1998fast} and became the mainstream fashion. In the SVD model, the Hopkins diffraction model can be decomposed into a sum of coherent systems based on eigenvalue decomposition:
\begin{equation}
    \mathbf{I}(x,y) = \sum_{k = 1}^{ N^{2}} w_{k} | \mathbf{M}(x,y) \otimes h_{k}(x,y) |^{2}, \quad x,y = 1,2,...N
\end{equation}
where $h_{k}$ is the $k$-th kernel and $w_{k}$ is the corresponding weight of the coherent system. "$\otimes$" denotes the convolution operator. \cite{gao2014mosaic} indicates the $K$-th order approximation:
\begin{equation}
    \mathbf{I}(x,y) \approx \sum_{k = 1}^{K} w_{k} | \mathbf{M}(x,y) \otimes h_{k}(x,y) |^{2},
\end{equation}

We pick $K = 24$ in our experiment. After optical simulation, the lithography intensity $\textbf{I}$ is sent to the photoresist model to generate the final binary pattern $\mathbf{Z}$ with an exposure resist threshold $I_{th}$:
\begin{equation}
    \mathbf{Z}(x,y) = 
    \begin{cases}
        1, & \text{if} \quad \mathbf{I} (x,y) \geq I_{th}, \\
        0, & \text{if} \quad \mathbf{I} (x,y) < I_{th}, 
    \end{cases}
\end{equation}

Several machine learning-based lithography simulation methods have been proposed. 
\cite{watanabe2017accurate} utilized a CNN network to perform a function model determination for resist model simulation. 
\cite{ye2019lithogan} developed a GAN-based LithoGAN, to map the input mask and output resist pattern. 
\cite{shao2020ic} proposed a two-stage DNN-based framework, solving the mask-to-SEM prediction as a domain-transfer problem and using CycleGAN \cite{zhu2017unpaired} to learn the transferring process.

Although DNN models usually have the comparative speed advantage, we choose the Hopkins model for the reason of analyzability. A white box model enables us to analyze the pattern shift equivariance property mathematically during the lithography process.

\subsection{OPC Evaluation Criteria}
\begin{figure}
    \subfloat[]{ \label{fig:epe} \includegraphics[width=.45\linewidth]{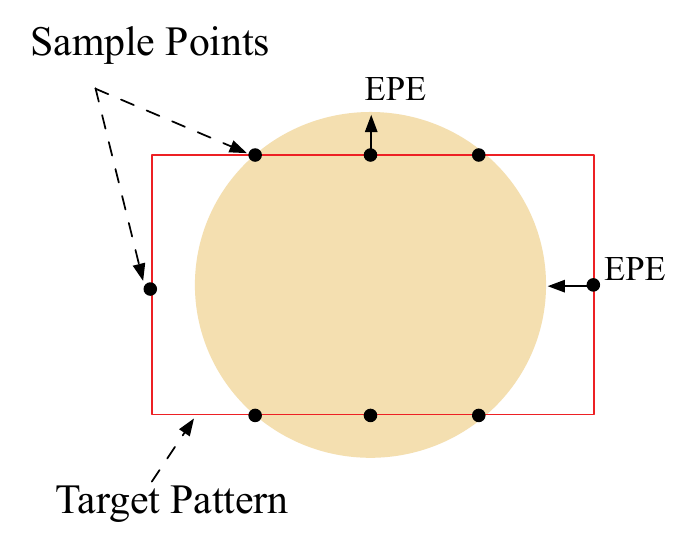} } 
    \subfloat[]{ \label{fig:pvband} \includegraphics[width=.45\linewidth]{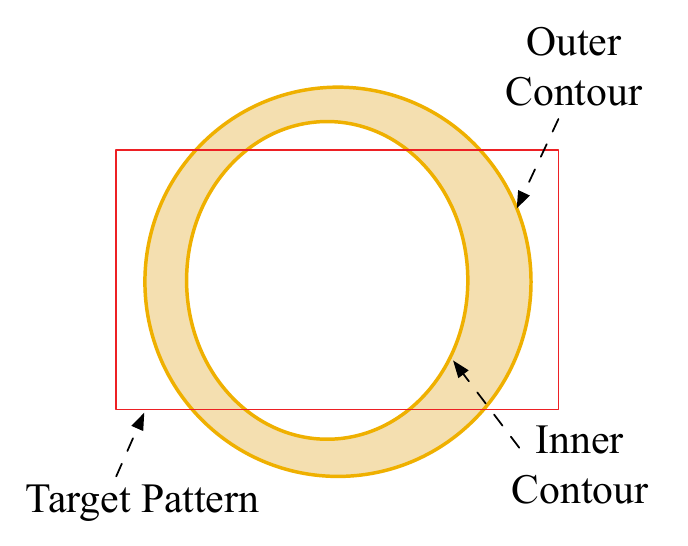} }
    \caption{OPC evluation creteria: (a) Visualization of EPE measurement (b) Visualization of PVBand.}
\end{figure}

\subsubsection{Edge placement error (EPE).}
\enspace After the lithography process, the printed image on the wafer has an inevitable geometric distortion from the design target. Edge placement error (EPE) is a common criterion to quantify distortion level. 
Measurement of EPE is visualized in \Cref{fig:epe}: A series of measuring points are sampled along the boundary of the target design pattern, including vertical edges and horizontal edges. If the distance $D$ between printed image and target is larger than threshold $th_{EPE}$ at a sample point, we label it as a EPE violation.
\begin{equation}
    EPE\_violation(x,y) = 
    \begin{cases}
        1, & D(x,y) \geq th_{EPE}, \\
        0, & D(x,y) \leq th_{EPE},
    \end{cases}
\end{equation}

\subsubsection{Process Variation Band (PV Band).}
\enspace In real lithography applications, process variation may cause deviation in the final printed images, which possibly leads to printing failure. Given different lithography conditions such as focus/defocus depth and incident light intensity, printed images have various contour results. Process Variation Band (PV Band) is defined as discrepant (XOR) region of innermost and outermost contours as shown in \Cref{fig:pvband} to evaluate printing robustness.
\begin{equation}
    PVBand = \sum_{x,y}^{N^2} | \textbf{Z}_{out} - \textbf{Z}_{in} | ,
\end{equation}
where $N$ is the size of pattern. $\textbf{Z}_{out}$ denotes the printed pattern of outer contour and $\textbf{Z}_{in}$ denotes the inner contour. 

\section{Adaptive Framework}
\begin{figure*}
    \centering
    \includegraphics[width=0.99\linewidth]{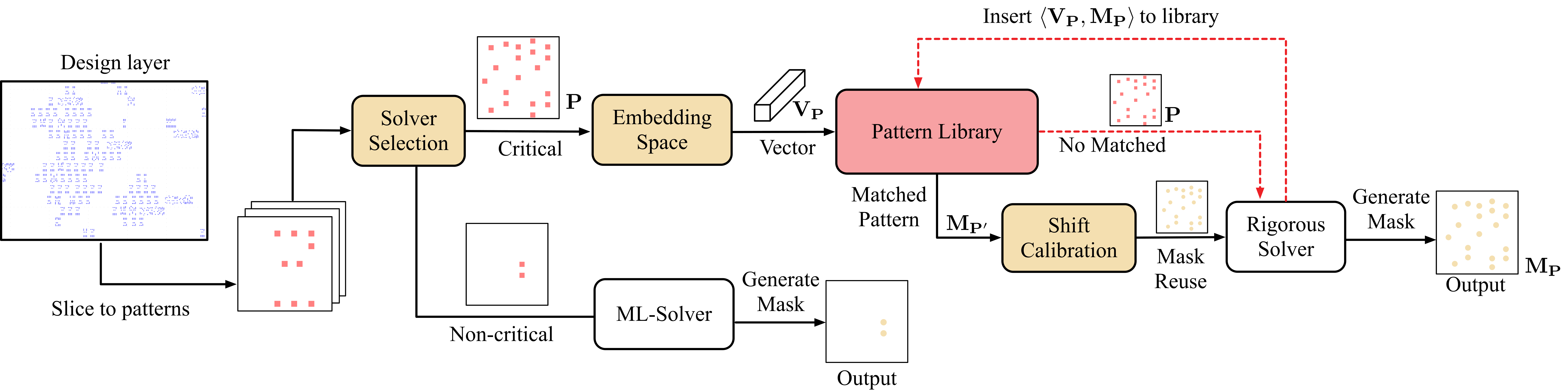}
    \caption{Overall workflow of AdaOPC. Colored blocks are functional modules. Red dashed lines represent library update flow.}
    \label{fig:workflow}
\end{figure*}
\subsection{Workflow Overview}
Our proposed workflow is visualized in \Cref{fig:workflow}.
In \Cref{solverpool}, we firstly introduce the extensible solver selection module to choose OPC solver for different patterns. 
In \Cref{library} we will demonstrate the dynamic graph-based Pattern library, one of the key components of AdaOPC framework for mask Optimization efficiency improvement. 
An approximate nearest neighbor searching(ANNS) method will be utilized to match similar patterns. 
\Cref{embedding} specified how we embed patterns into high dimensional vectors for pattern matching in the library using supervised contrastive learning.
\Cref{shift} will further discuss the mask reusability and requirement. A proof of the shift equivariance during lithography is provided to validate the feasibility, along with the solution by shift calibration.


\subsection{Extensible OPC Solver Selection} \label{solverpool}
Our framework maintains a flexible solver pool to select suitable OPC solutions for different patterns based on their complexity. We divide sliced design patterns into two groups: critical and non-critical patterns. Then we use a solver selector to choose which OPC solver to use. This Solver selector can be regarded as a 2-class classifier and built with a simple deep learning classification model. We use Resnet-18 \cite{he2016deep} as backbone network with objective to minimize the cross-entropy loss $L$:
\begin{equation}
    L = - \frac{1}{N} \sum_{i}^{N} y_{i}\log(p_{i}) + (1 - y_{i})\log(1 - p_{i}) ,
\end{equation}
where $y_{i}$ is the 1/0 label for critical pattern of sample $i$, $p_{i}$ is the probability predicted by classifier model. The model is trained to predict $p_{i}$ as close as $y_{i}$. Without bells and whistles, such simple network + loss combination is capable of performing a fast and accurate prediction on pattern class.

For non-critical patterns, we build our ML-Solver using a generative neural network model, consistent with one of the SOTA ML-based OPC solvers DAMO-DMG \cite{chen2021damo}. We also pick U-Net++ with residual blocks inserted in the bottleneck as our model structure and train the generative model with the same strategy as \cite{chen2021damo}. The only difference is the training data. We prepare our own training dataset with patterns from a real full-scale design and mask generated by a robust OPC engine, of which the lithography model is an authentic one instead of a DNN simulator as in \cite{chen2021damo}. Such data preparation aligns with the real OPC scenario, where the only ground truth we have is the lithography model.

For critical patterns, we use rigorous numerical optimization method as \cite{gao2014mosaic} with GPU acceleration by CUDA, despite that deep learning approaches already reached good performance for an overall evaluation on some test sets of patterns. Data-driven black-box deep learning model may learn to mimic and reverse the diffraction effect very well. Nevertheless, they might have difficulty dealing with optical interference of incident light caused by complex neighboring components. 
In that case, the rigorous numerical optimization solver provides an analytical solution regardless of the geometric complexity of patterns.
Moreover, in a real OPC scenario with a new design and even possibly a new lithography engine, patterns and optimized masks generated by robust methods can be the dataset to train the machine learning model to adapt to new settings. 

Note that the solver pool is extensible. Any powerful OPC solution with certain strengths for certain patterns holds the possibility to be imported as a replacement or complementary candidate. If more than two solvers are in the solver pool, we can simply modify the classifier loss:
\begin{equation}
    L = -\frac{1}{N}\sum_{i}^{N}\sum_{c=1}^{C} y_{ic}\log(p_{ic}) ,
\end{equation}
where $C$ is the number of pattern classes, the same as the number of corresponding OPC solvers. $y_{ic}$ is the 1/0 label for whether this pattern belongs to class $c$. In this way, we can simply transform the problem into a multi-classification case.
\section{Dynamic Pattern Library}\label{library}
After filtering out simple cases, we need to apply a rigorous solver for critical patterns. For further efficiency improvement, we build a dynamic pattern library to store the pattern pairs: sliced pattern $\mathbf{P}$ and their corresponding post-OPC mask $\mathbf{M}_{\mathbf{P}}$, hence to reuse the mask of repeating patterns to avoid the redundant time consumption of iterations-from-scratch OPC process. The key idea is to identify a stored repeating pattern before OPC. An online update mechanism enables a brand new pattern and corresponding mask to be inserted into the library.

We construct the pattern library with a graph structure as inspired by \cite{malkov2018efficient}, where each node represents a pattern stored. Each edge connecting two nodes shows they are neighbors of each other with more similarity. Considering the number of patterns on a full design layout, the size of the graph is gigantic. Naive search for shortest distance requires pair-wise distance comparisons of nodes in the graph, which is impractical. We improve the matching time efficiency with:
\begin{itemize}
    \item Sparse neighborhood graph structure, where nodes distant from each other are sparsely connected.
    \item Graph is divided into hierarchical layers. Nodes have restricted degree at each layer. more edges reside at lower layers, enabling greedy search of nearest neighbor at each layer.
\end{itemize}
The visualization of the hierarchical sparse graph structure is in \Cref{fig:hnsw}. 

\subsection{Pattern Matching And Online Update.}
The task of matching pattern with the most similar geometric shape can be regarded as a nearest neighbor search problem (NNS). As inspired by \cite{malkov2018efficient}, we utilize the Hierarchical Navigate Small World (HNSW) algorithm for fast matching. As shown in \Cref{fig:hnsw}, the pattern matching follows a greedy strategy to traverse the graph from higher layers to the bottom layer. A list of nearest pattern node candidates is maintained through the top-down traversal. The list is updated when a closer pattern appears during searching that has a distance shorter than one of the candidates. Such matching strategy is based on proximity graph nearest neighbors search. The detailed pattern matching search strategy at each hierarchical layer is illustrated in \Cref{alg:matching Strategy}.
\begin{algorithm}
    \small
    \caption{Graph-Based Pattern Matching Greedy Search}
    \label{alg:matching Strategy}
    \begin{algorithmic}[1]
    \State {\textbf{Input:} query pattern $\mathbf{P}$, starting nodes $q_{s}$, number of nearest neighbor to return $k$,layer number $l$, distance measurement $d(\cdot)$} 
    \State {\textbf{Output:} nearest pattern candidates $C$}
    \State {$V \leftarrow q_{s}$} // Visited nodes
    \State {$W \leftarrow q_{s}$} // Waiting list of nodes to visit
    \State {$C \leftarrow q_{s}$}
    \While {$\lvert W \rvert > 0$}
        \State {$q^{*} \leftarrow$ nearest pattern from $W$ to $\mathbf{P}$}
        \State {$q_{f} \leftarrow$ furthest pattern from $C$ to $\mathbf{P}$}
        \If {$d(\mathbf{P}, q^{*}) > d(\mathbf{P}, q_{f})$}
        \State {break}
        \EndIf
        \For {$e \in neighbor(q^{*})$ in layer $l$}
        \If {$e \not\in V$}
        \State {$V \leftarrow V \cup \{e\}$}
        \State {$q_{f} \leftarrow$ furtherest pattern from $C$ to $\mathbf{P}$}
        \If {$d(\mathbf{P}, e) < d(\mathbf{P},q_{f})$ or $\lvert C \rvert < k$}
        \State {$W \leftarrow W \cup \{e\}$}
        \State {$C \leftarrow C \cup \{e\}$}
        \If{$\lvert C \rvert > k$}
        \State {remove furthest pattern from $C$ to $\mathbf{P}$} 
        \EndIf
        \EndIf
        \EndIf 
        \EndFor
    \EndWhile
    \end{algorithmic}
\end{algorithm}

After reaching the bottom layer, pattern candidates in list $C$ with a distance smaller than a threshold $\sigma$ are regarded as matched patterns. If the smallest distance in $C$ is still larger than $\sigma$, we regard it as a new pattern. This approach remains fast and accurate even when the graph grows large with new patterns continuously being inserted in the library.

The pattern library updates in an online style. For every new pattern that comes with no matched ones, the mask is optimized with from-scratch OPC iterations. The library will insert the pattern and optimized mask as new node and update the edge hierarchy of the graph to store the pattern. \Cref{alg:library update} shows the library online update step details.
\begin{figure}
    \includegraphics[width=0.98\linewidth]{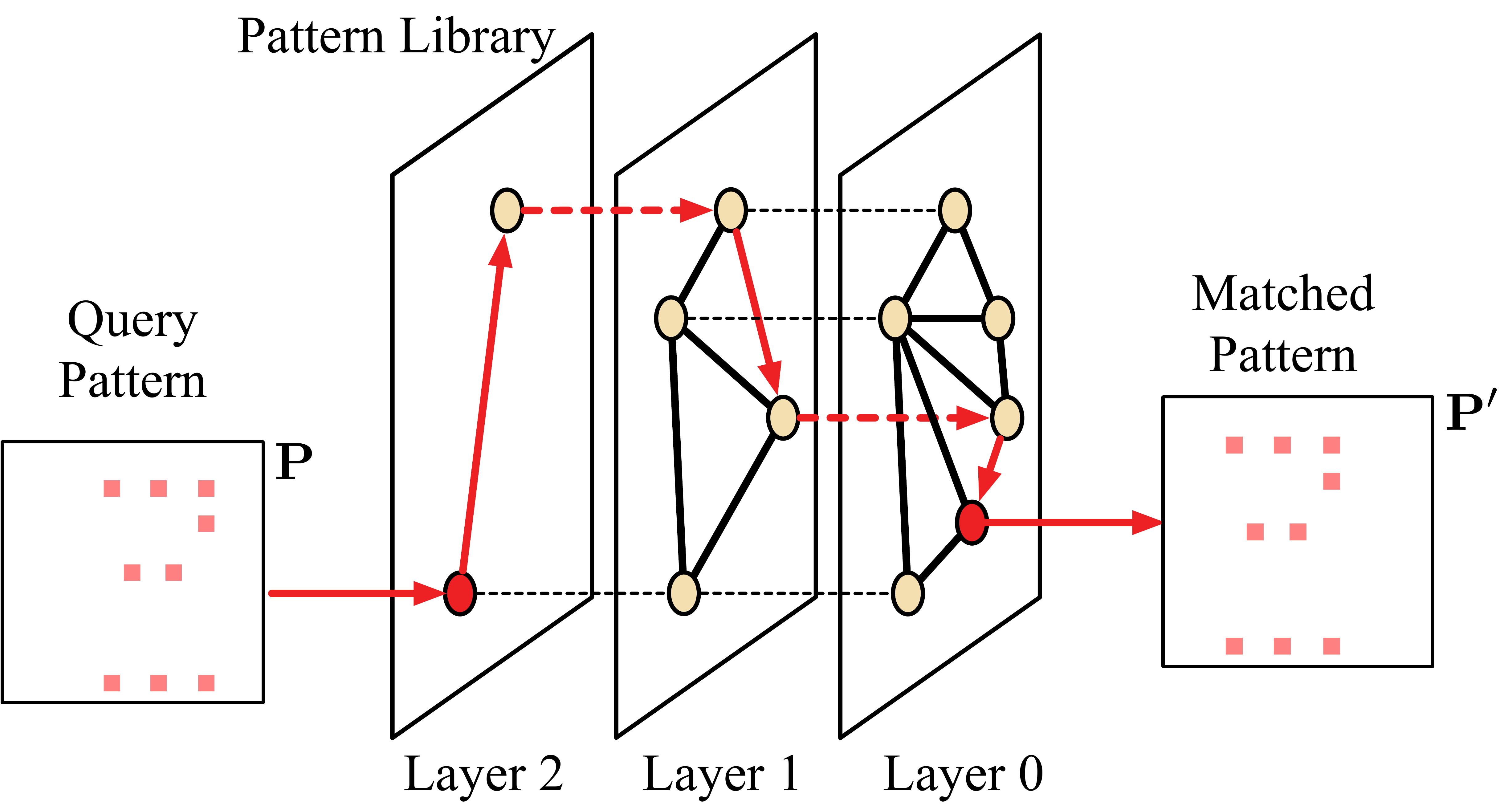} 
    \caption{Visualization of the graph-based pattern matching flow. Query design pattern $\mathbf{P}$ greedily traversing the hierarchical graph. The nearest node reached at layer 0 corresponds to a match pattern $\mathbf{P'}$ which has the most similar geometric shape with $\mathbf{P}$.}
    \label{fig:hnsw}
\end{figure}
\begin{algorithm}
    \small
    \caption{New Pattern Insertion and Graph Update}
    \label{alg:library update}
    \begin{algorithmic}[1]
        \State {\textbf{Input:}hierarchical graph $G$, new pattern $\mathbf{P}$, total layer number $L$, $G$'s starting nodes $q_{s}$, max degree $M$}
        \State {\textbf{Output:} updated hierarchical graph $G$}
        \State {$l \leftarrow random(0,L)$ //exponentially decaying probability}
        \For {$l_{c} \leftarrow L,..l$}
        \State {$C \leftarrow search(P, q_{s}, k, l_{c})$} \Comment{\Cref{alg:matching Strategy}}
        \State {$q_{s} \leftarrow$ nearest pattern of $q$ in $C$}
        \EndFor
        \For {$l_{c} \leftarrow l,.. 0$}
        \State {insert $\mathbf{P}$ to layer $l_{c}$ of $G$ // add $\mathbf{P}$ into graph} 
        \State{C $\leftarrow search(P, q_{s}, k, l_{c})$ } \Comment{\Cref{alg:matching Strategy}}
        \State{$ neighbors(\mathbf(P)) \leftarrow$ top $M$ nearest patterns in $C$}
        \For {$e \leftarrow neighbors(\mathbf{P})$}
        \State {add edge $(P, e)$}
        \If {degree of $e > M$}
        \State {$neighbors(e) \leftarrow$ top k nearest patterns connecting $e$}
        \State {remove all edges connecting $e$}
        \State {create edges $e$ with each one in $neighbors(e)$}
        \EndIf
        \EndFor
        \EndFor 
    \end{algorithmic}
\end{algorithm}

As \Cref{alg:library update} indicates, the new pattern will be inserted to one of the hierarchical layers with decaying probability. Edges will be added at the same layer between this pattern and top $k$ nearest ones. As the neighbor nodes will have degree increase, an edge re-connection of these neighbors will be conducted once degree is above the upper bound $k$. Therefore, the degree of each node in the graph is restricted by $k$. Note that the number of edges results in the matching complexity. Such a sparse hierarchical graph enables a fast search when the size of graph grows up. 

Since vectors of the same pattern are closer to each other, we proposed some distance metrics to evaluate the similarity of vectors. The inner product of two vectors can be used to evaluate the direction difference of vectors:
\begin{equation}
    Inner(\mathbf{V}_{\mathbf{P}_{1}},\mathbf{V}_{\mathbf{P}_{2}}) = \mathbf{V}_{\mathbf{P}_{1}} \cdot \mathbf{V}_{\mathbf{P}_{2}} = \sum_{i = 0}^{k} V_{\mathbf{P}_{1},i} V_{\mathbf{P}_{2},i} ,
\end{equation}
where $\mathbf{V}_{\mathbf{P}_{1}}$ and $\mathbf{V}_{\mathbf{P}_{2}}$ are the embedded vectors of pattern $\mathbf{P}_{1}$ and $\mathbf{P}_{2}$. $k = 256$ is the dimension of the embedded vector. Note that the inner product violates the positivity property where an element can be closer to some other element than to itself. 

\textbf{Consine similarity} avoid the violation thus can be utilized to measure the similarity between two vectors of an inner product space:
\begin{equation}
\begin{aligned}
    d_{Cosine}(\mathbf{V}_{\mathbf{P}_{1}},\mathbf{V}_{\mathbf{P}_{2}})
    &= 1.0 - \frac{\mathbf{V}_{\mathbf{P}_{1}} \cdot \mathbf{V}_{\mathbf{P}_{2}}}{ \left\lVert\mathbf{V}_{\mathbf{P}_{1}}\right\lVert \left\lVert\mathbf{V}_{\mathbf{P}_{2}}\right\lVert}, \\[6 pt]
    &= 1.0 - \frac{\sum_{i = 0}^{k} V_{\mathbf{P}_{1},i} V_{\mathbf{P}_{2},i}}{\sqrt{\sum_{i = 0}^{k} V_{\mathbf{P}_{1},i}^2 } \sqrt{\sum_{i = 0}^{k} V_{\mathbf{P}_{2}, i}^2}},
\end{aligned}
\end{equation}

\textbf{Euclidean Distance} is another approach where the embedding metric space is regarded as a Euclidean Space. Each vector represents its position in the Cartesian coordinates. The similarity of two vectors can be evaluated directly by calculating the squared $\mathscr{l}$-2 norm of the difference of the coordinates: 
\begin{equation}
    d_{Euclid}(\mathbf{V}_{\mathbf{P}_{1}},\mathbf{V}_{\mathbf{P}_{2}}) = \left\| \mathbf{V}_{\mathbf{P}_1} - \mathbf{V}_{\mathbf{P}_2} \right\|_{2}^{2} = \sqrt{\sum_{i=0}^{k} (V_{\mathbf{P}_{1},i} - V_{\mathbf{P}_{2},i})^2 }.
\end{equation}
All metrics abide by the rules of nearest neighbor search (NNS) are feasible similarity measurement metrics, which leave possibility for exploration of various metrics according to different embedding space. 

\subsection{Embedding Space Construction}\label{embedding}
To reuse mask stored in library we need to match a pattern with same geometric shape. However it is not straighforward to compare geometric similarity of two patterns directly. We develop a embedding metric space which reflects the geometric property using a high-dimensional vector representation $\mathbf{V}_{\mathbf{P}}$. Original $\langle \mathbf{P}, \mathbf{M}_{\mathbf{P}}\rangle$ pair stored in library is replaced with $\langle \mathbf{V}_{\mathbf{P}}, \mathbf{M}_{\mathbf{P}} \rangle$ pair. In this way, a decision of whether two patterns are the same can be determination by a similarity metric of two embedded vectors.

We transform the embedding space construction into a feature extraction process by using deep learning model, and the metric space is built through deep metric learning. The embedded vector is the output of an embedding neural network. Deep learning model for the embedding process is composed of two modules: 
\begin{itemize}
    \item  \textit{Encoder}, $Enc(\cdot)$. For each input pattern $\mathbf{P}$, the encoder will encode the input pattern to a feature map $\mathbf{F}_{P} \in \mathbb{R}^{h\times w\times c}$, where $h, w$ are the spatial size of the feature map $\mathbf{F}_{P}$ and $c$ is the number of channels. 
    \item \textit{Projector}, $Proj(\cdot)$, which embeds feature map $\mathbf{F}_{P}$ to a representation vector $\mathbf{V}_{\mathbf{P}} \in \mathbb{R}^{k}$. The output $Proj(\mathbf{F}_{P})$ is normalized to the unit hypersphere in $\mathbb{R}^{k}$ at training stage for loss calculation. 
\end{itemize}
Therefore, the embedding process is formulated as:
\begin{equation}
    \mathbf{V}_{\mathbf{P}} = Proj(Enc(\mathbf{P})) \in \mathbb{R}^{k} .
\end{equation}
Previous deep learning-based OPC approaches \cite{yang2019gan, ye2019lithogan, chen2021damo} chose UNet or special variant UNet++ as the backbone structure. A harsh requirement of OPC problem strongly limits the selection of network backbone structure: the output mask will necessarily remain the same resolution as the input design.  

Embedding process without such limitation leaves us the flexibility to more network structure candidates.
We deliberately choose one of the most common structures: Resnet-18 \cite{he2016deep} as encoder. Each input pattern $\mathbf{P}$ is a $2048 \times 2048$ 2-D picture. 
In order to restrict the heavy computation caused time delay, we downsample the pattern  in a greedy manner until 256*256 before sending it into the neural network, without noticeable performance degradation. 
After the original Resnet-18 structure, an extra $1\times1$ convolution layer follows to shrink the feature channel size from 512 to 256.
A linear layer is put at the end of the neural network to transform the 3-D feature into the final 1-D embedded vector $\mathbf{V}_{\mathbf{P}}$. 
The size of $\mathbf{V}_{\mathbf{P}}$ is a trade-off where larger size indicates higher matching accuracy but slower similarity computation and matching speed. 
We select 256 through experiments to guarantee good performance with neglectable matching time.

The embedding space $\mathbf{S}$ is specially designed with certain objectives:
\begin{enumerate*}
    \item patterns of the same shape share similar embedded vectors with the shortest distance.
    \item patterns of different shapes clustered sparsely in the embedding space, far from each other.
\end{enumerate*}
The training process of such embedding requires abundant data of different patterns as well as data of the same patterns. During training, the data from the same pattern is regarded as positive samples, and the embedded vectors will be pushed as close as possible with higher similarity, while the embedded vectors of different patterns are regarded as negative samples pushed as far as possible. We crop a great amount of patches of pattern from a real full-scale design for the training dataset. 

\textbf{Data Preparation.} 
The cropping process has two steps to generate positive samples and negative samples as required by the training objective. The first step is to randomize some anchor points along the design layer. In the second step, with a random shift around each anchor point, we crop a certain number of patches, which will hold the same pattern within the square patch but with different relative positions. For each batch of training data, patterns will be labeled by their anchor point and thus abides by the positive/negative samples requirement. 

\textbf{Supervised Contrastive Loss.}
As we draw the picture of training the neural network to learn how to embed patterns to representative vectors, the traditional cross-entropy loss is not sufficiently sensitive to handle inter-class distance or noise labels. Among the family of losses based on metric distance learning \cite{hadsell2006dimensionality, wu2018unsupervised, hjelm2018learning}, Contrastive loss \cite{chen2020simple} is one of the most powerful losses for learning representative embedding in the self-supervised learning domain. Inspired by \cite{khosla2020supervised}, we extend the contrastive loss to supervised contrastive loss as all the positive/negative samples are genuinely generated and labeled at data preparation stage. As mentioned, representation vector $\mathbf{z}$ are normalized from $\mathbf{V}_{\mathbf{P}}$:
\begin{equation}
    \mathbf{z} = \text{normalize}(Proj(Enc(\mathbf{P})) ) \in \mathbb{R}^{k}  ,
\end{equation}
Then the loss function is formulated as:
\begin{equation} 
    \\[2pt]
     \mathcal{L}_{supCon} = -\sum_{i \in I} \frac{1}{|J(i)|} \sum_{j \in J(i)} \log \frac{\exp (\mathbf{z}_{i} \cdot \mathbf{z}_{j}/\tau)}{ \sum_{a \in A(i)} \exp(\mathbf{z}_{i} \cdot \mathbf{z}_{a} / \tau) } ,
\end{equation}
where $i \in I$ is anchor indices of the training batch. $j \in J$ is the anchor indices of the positive samples. $A(i) = I \backslash \{i\}$ is all anchor indices except for $i$ in this batch and therefore $A(i) \backslash \{J(i)\}$ are the anchor indices of negative samples. $\tau$ is a scalar temperature parameter. Term $\exp (\mathbf{z}_{i} \cdot \mathbf{z}_{j}/\tau)$ in numerator denotes similarity of positive sample pairs $\mathbf{z}_{i}$ and $\mathbf{z}_{j}$. $\exp (\mathbf{z}_{i} \cdot \mathbf{z}_{a}/\tau)$ in denominator denotes similarity of all sample pairs including negative ones. By minimizing the loss, the training process enlarges the similarity of positive samples and reduces the similarity of negative samples.

\section{Mask Reuse With Shift Calibration}\label{shift}

\subsection{Mask Reusability}
We assume that repeating patterns can share mask for efficiency if the query design pattern $\mathbf{P}$ has a matched design pattern $\mathbf{P'}$ stored in the library with the same shape. However, \Cref{fig:slicing} has shown that when the whole design is sliced into small patterns, it is inevitable to find pattern location shift $(\Delta x, \Delta y)$ between $\mathbf{P}$ and $\mathbf{P'}$. In real lithography and OPC flow, if no surrounding factors affect lithography, the printed wafer image patch must have no distortion but only an identical shift to the design pattern, as visualized in \Cref{fig:shift}. Therefore, if we want to reuse the mask, the first requirement is to make sure that location shift during lithography will not cause any geometric distortion.

We mathematically prove the location shift remains unchanged before and after the lithography in order to show the feasibility of the mask shift calibration approach, we denote the Hopkins diffraction model through lithography in \Cref{litho_model} as $Litho(\cdot)$ and the location shift as $\delta_{\Delta x, \Delta y}(\cdot)$, we show that:

\begin{mytheorem}[Shift Equivariance]
    \label{shift_theorem}
    Given pattern $\mathbf{P}$ and mask $\mathbf{M}_{\mathbf{P}}$ where
    \begin{equation}
        \mathbf{P} = Litho(\mathbf{M}_{\mathbf{P}}) .
    \end{equation}
     The following statement always holds: 
    \begin{equation}
        \label{litho_equivariance}
        \delta_{\Delta x, \Delta y}(\mathbf{P}) = Litho(\delta_{\Delta x, \Delta y}(\mathbf{M}_{\mathbf{P}})).
    \end{equation}
\end{mytheorem}
\begin{proof}
    For any position (x,y) on pattern $\mathbf{P}$:
    \begin{equation}
        \begin{aligned}
            &\delta_{\Delta x, \Delta y}(\mathbf{P}(x,y)) = \mathbf{P}(x + \Delta x,y + \Delta y), \\
            &= \sum_{k = 1}^{N^2} w_{k} \left| h_{k}(x + \Delta x,y + \Delta y) \otimes \mathbf{M}_{\mathbf{P}}(x + \Delta x,y + \Delta y) \right|^{2}, \\
            &= \sum_{k = 1}^{N^2} w_{k} | \sum_{i=1}^{N} \sum_{j=1}^{N} h_{k}(i,j) \mathbf{M}_{\mathbf{P}}(x + \Delta x + i - \frac{N}{2}, y + \Delta y + j - \frac{N}{2})|^{2}, \\[4pt]
            &= \sum_{k = 1}^{N^2} w_{k} | \sum_{i=1}^{N} \sum_{j=1}^{N} h_{k}(i,j) \mathbf{M}_{\mathbf{P}}(x + i - \frac{N}{2}+ \Delta x, y + j - \frac{N}{2} + \Delta y)|^{2}, \\[4pt]
            & = \sum_{k = 1}^{N^2} w_{k} | \sum_{i=1}^{N} \sum_{j=1}^{N} h_{k}(i,j) \delta_{\Delta x, \Delta y}(\mathbf{M}_{\mathbf{P}}(x + i - \frac{N}{2}, y + j - \frac{N}{2}))|^{2}, \\[4pt]
            & = \sum_{k = 1}^{N^2} w_{k} \left| h_{k}(x,y) \otimes \delta_{\Delta x, \Delta y}(\mathbf{M}_{\mathbf{P}}(x,y)) \right|^{2}, \\[4 pt]
            & = Litho(\delta_{\Delta x, \Delta y}(\mathbf{M}_{\mathbf{P}}(x,y))) .
        \end{aligned}
    \end{equation}
    Then \Cref{litho_equivariance} is proved.
\end{proof}
\begin{figure}
    \includegraphics[width=0.95\linewidth]{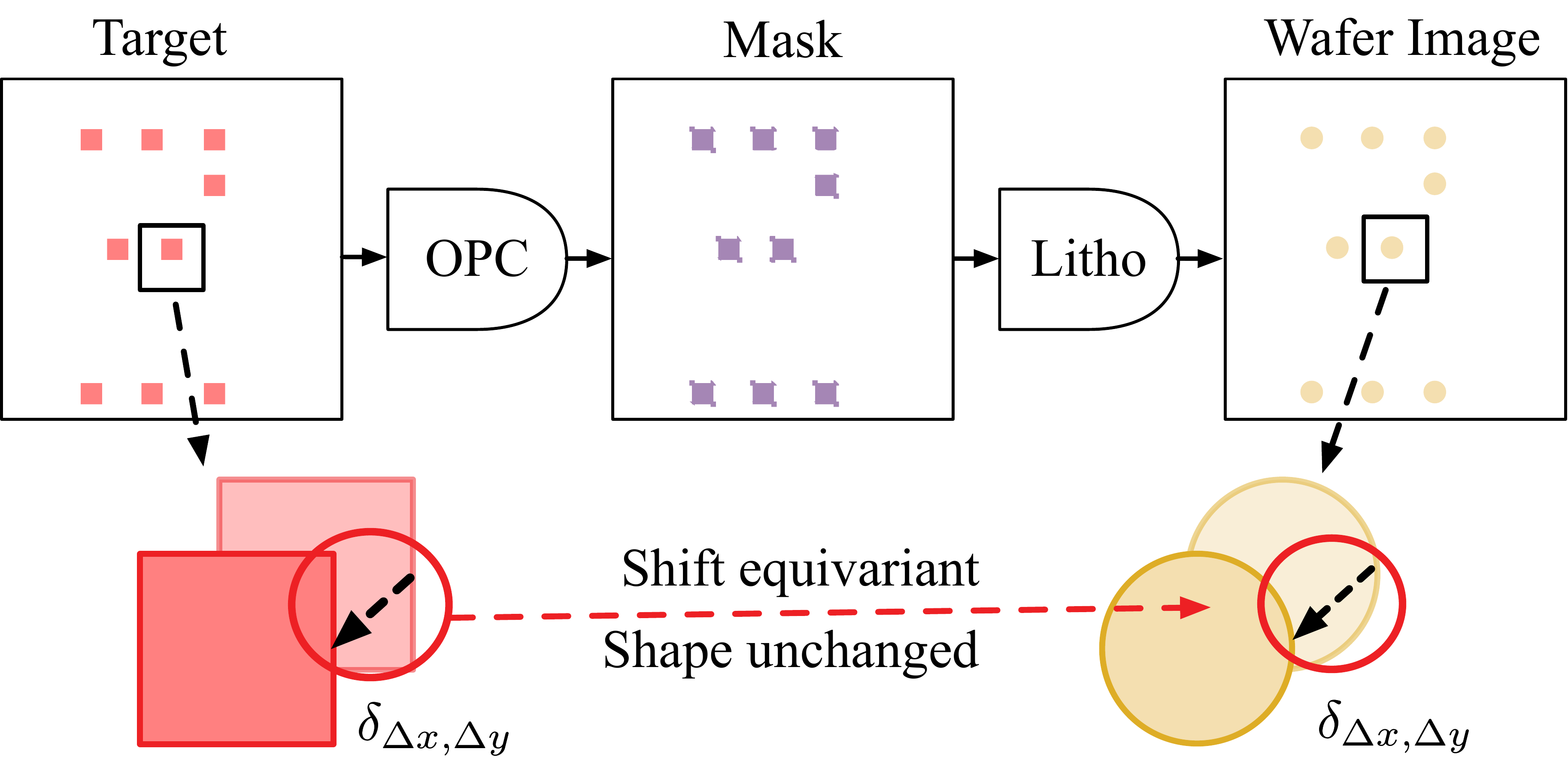} 
    \caption{Printed wafer image must share identical location shift to design pattern with no geometric shape distortion.}
    \label{fig:shift}
\end{figure}

Since mask shift will only result in a printing shift after lithography, repeating patterns in design can share OPC-optimized masks with a simple shift correction. We pick matched mask $\mathbf{M}_{\mathbf{P'}}$ stored in pattern library and add correction $(-\Delta x, -\Delta y)$ to acquire the initial mask $\mathbf{M}_{\mathbf{P}}$ for $\mathbf{P}$.

\subsection{Pattern Shift Calibration}
We calculate the shift by computing the pixel-level similarity of two patterns  $\mathbf{P}$ and $\mathbf{P'}$. The pixel-wise \textbf{cross-correlation} of $\mathbf{P}$ and $\mathbf{P'}$ reflects the pixel-wise similarity where the pixel of highest response value on the correlation map is the position shift of center point $(x_{ctr}, y_{ctr})$. The cross-correlation computation of two large 2-D pattern is time-comsuming. The calculation process of cross-correlation equal to convolution of $\mathbf{P}$ and $Rotate(\mathbf{P'})$, where $Rotate(\cdot)$ denotes rotation of $180^{\circ}$:
\begin{equation}
    CrossCorr(\mathbf{P}, \mathbf{P'}) = Conv(\mathbf{P}, Rotate(\mathbf{P'})),
\end{equation}
We replace the calculation with Convolution and accelerate the computation with Fast Fourier Transform (FFT) \cite{vasilache2014fast}. The pattern shift can be calculated with:
\begin{equation}
    \begin{aligned}
        x^{*}, y^{*} &= \argmax_{x,y}\  Conv\_FFT (\mathbf{P}, Rotate(\mathbf{P'})), \\
        \Delta x &= x^{*} - x_{ctr},\  \Delta y = y^{*} - y_{ctr},
    \end{aligned}
\end{equation}
And the initial mask is corrected with:
\begin{equation}
    \mathbf{M}_{\mathbf{P}} = \delta_{-\Delta x, -\Delta y}(\mathbf{M}_{\mathbf{P'}}).
\end{equation}
In real application, we send calibrated mask into lithography model to verify the mask and into ILT solver for one or two further iterations if necessary, just in case of any noise caused by shift calibration operation.
We use the same pattern size as \cite{yang2019gan} $2048 \times 2048$. With our implementation, the shift calculation time is less than 0.25s on CPU.  

\section{Experimental Results}

Our framework is mainly developed in Python. All machine learning modules in our framework are implemented using PyTorch. Lithography and ILT modules are developed in C/C++ with CUDA toolkit. All performance and speed experiments are conducted on CentOS-7 system with Intel i7-5930K 3.50GHz CPU and Nvidia GTX Titan X GPU. We choose the public lithography engine from ICCAD 2013 CAD Contest \cite{banerjee2013iccad} with the 24 optical kernels. The photoresist intensity threshold is set at 0.055. We adopt lithography wavelength $193nm$ with defocus range of $\pm 25nm$ and dose range of $\pm 2\%$. EPE violation threshold $th_{EPE}$ is set to $15nm$.

In accordance with the authentic OPC scenario, all data used in our experiments is from a real design which we extracted from a gds file generated by the open-source layout generation tool OpenROAD \cite{ajayi2019openroad}. We sliced patterns of size $2048\times 2048$ in alignment with previous work~\cite{yang2019gan,yu2021gpu,chen2021damo} from a full scale via layer with more than $1.9 \times 10^6$ vias, where each pixel represents $1nm^2$. 
For the training dataset of ML-Solver, we randomly slice 4000 patterns off the design layer with masks optimized by ILT Solver. We use the same patterns to train the pattern classifier. As for critical/non-critical labels, we directly apply lithography and label them by EPE number, which is intuitive as it reflects the mask optimization difficulty. 
For the training data of Metric Space embedding, we follow the steps in \Cref{embedding}, using 400 random anchor points and shift around each anchor point within $\pm10\%$ pattern width for 400 repeating patterns. The number of positive samples for each point should be equal to or larger than the number of anchors to guarantee a positive/negative ratio of each training batch.

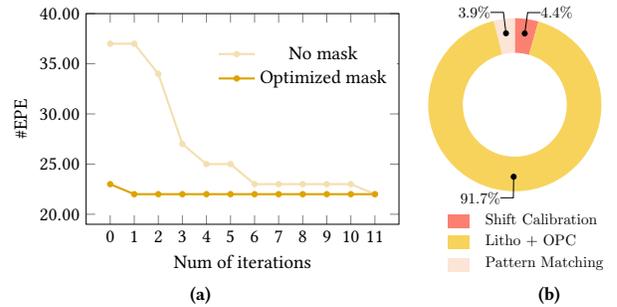
\begin{figure}
    \hspace*{-0.3cm}\subfloat[]{ \label{fig:epe1} \pgfplotsset{
    width = 0.8\linewidth,
    height= 0.6\linewidth
}

\begin{tikzpicture}[scale=0.8, style={outer sep=0}]
\begin{axis}[
xmax=13, xmin=0,
ymax=40, ymin=19,
xtick={ 1, 2, 3, 4, 5, 6, 7, 8, 9, 10, 11, 12},
xticklabels={0, 1, 2, 3, 4, 5, 6, 7, 8, 9, 10, 11},
scaled ticks=false,
legend pos=north east,
xticklabel style={/pgf/number format/fixed},
yticklabel style={/pgf/number format/.cd, fixed, fixed zerofill, /tikz/.cd},
xlabel={Num of iterations},
ylabel={\#EPE},
xlabel near ticks,
ylabel near ticks,
legend style={
  draw=none,
  at={(0.4,0.75)},
  anchor=west,
  legend columns=1,
  scale = 1
}
]

\pgfplotstableset{
	create on use/x rel/.style={
		create col/expr={
			 \thisrow{0}
		}
	},
	create on use/y rel/.style={
		create col/expr={
			 \thisrow{1}
		}
	}
}

\addplot[CUribbon, line width=1pt, mark=*, mark options={scale=0.5, fill=CUribbon},text mark as node=true] table [x = x rel, y=y rel] {
    1 37
    2 37
    3 34
    4 27
    5 25
    6 25
    7 23
    8 23
    9 23
    10 23
    11 23
    12 22

};
\addlegendentry{No mask}

\addplot[CUgold, line width=1pt, mark=*, mark options={scale=0.5, fill=CUgold},text mark as node=true] table [x = x rel, y=y rel] {
    1 23
    2 22
    3 22
    4 22
    5 22
    6 22
    7 22
    8 22
    9 22    
    10 22
    11 22
    12 22
};
\addlegendentry{Optimized mask}

\end{axis}
\end{tikzpicture} } 
    \hspace*{-0.4cm}\subfloat[]{ \label{fig:breakdown} \includegraphics[width=.5\linewidth]{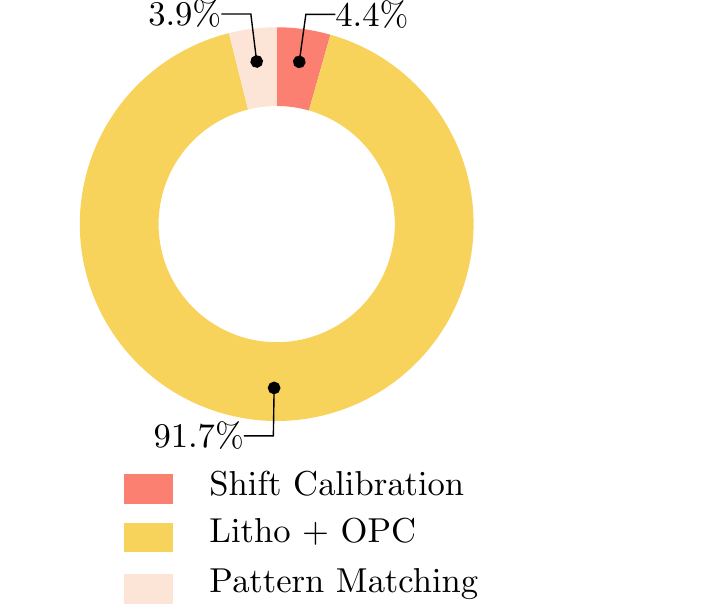} }
    \caption{(a) EPE convergence comparison (b) Runtime breakdown of AdaOPC on critical patterns}
\end{figure}

\begin{table}[tb!]
    \centering
   \renewcommand{\arraystretch}{0.9} 
    \begin{tabular}{c|ccc}
        \toprule
        Library Size & 128-D & 256-D & 512-D\\
        \midrule
        100 & 0.9ms& 1.3ms& 1.5ms\\
        500 & 3.4ms& 5.7ms& 8.8ms\\
        1000 & 9.0ms& 13.3ms& 22.2ms\\
        2000 & 20.4ms& 31.2ms& 52.2ms\\
        5000 & 59.8ms& 93.0ms& 156.6ms\\
        10000 & 130.1ms& 206.5ms& 413.6ms\\
        \bottomrule
    \end{tabular}
    \caption{Pattern matching speed Analysis on different embedding dimension.}
    \label{table:matching_time}
\end{table}

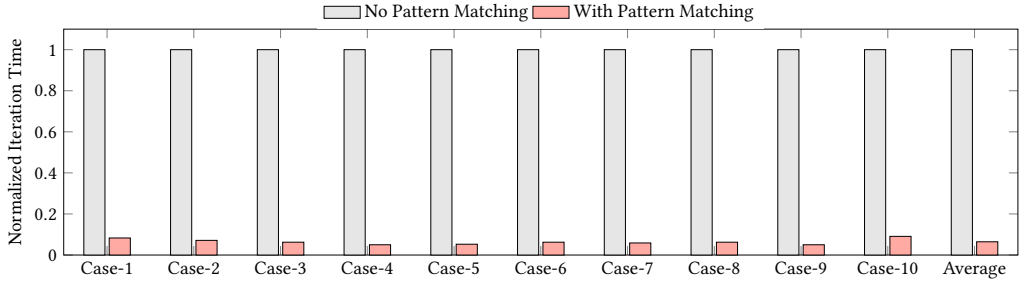
\begin{figure*}[tb!]
    \definecolor{mypink}{RGB}{255,171,164}     
\definecolor{mygray}{RGB}{240,240,240}     
\definecolor{lightblue}{RGB}{221,239,250}  
\definecolor{myred}{RGB}{228,46,36}        
\definecolor{mygreen}{RGB}{142,202,206}    
\definecolor{myyellow}{RGB}{255,253,181}   
\definecolor{myblue}{RGB}{71,162,241}      
\definecolor{myorange}{RGB}{244,106,18}
\definecolor{mylight}{RGB}{191,191,255}
\definecolor{darkblue}{HTML}{226E9C}
\definecolor{crystal}{HTML}{39B185}
\definecolor{crybrown}{RGB}{160,127,40}
\thispagestyle{empty}

\definecolor{CUpurple}{RGB}{136,43,142}
\definecolor{CUlpurple}{RGB}{165,133,182}
\definecolor{CUgold}{RGB}{221,163,0}
\definecolor{CUribbon}{RGB}{244,223,176}
\definecolor{CUblack}{RGB}{34,24,21}
\definecolor{PKUred}{RGB}{126,24,28}
\definecolor{gray6}{gray}{0.6}
\definecolor{gray7}{gray}{0.7}
\definecolor{gray8}{gray}{0.8}
\definecolor{gray9}{gray}{0.9}

\pgfplotsset{
    width =0.98\textwidth,
    height=0.30\textwidth
}

\begin{tikzpicture}[scale=0.8]
    \begin{axis}[
        ybar,
        xmin=0.5, xmax=11.5,
        ymin=0, ymax=1.1,
        xticklabels={Case-1, Case-2, Case-3, Case-4, Case-5, Case-6, Case-7, Case-8, Case-9, Case-10, Average},
        xtick={1, 2, 3, 4, 5, 6, 7, 8, 9, 10, 11},
        xtick align=inside,
        xticklabel style = {rotate=0}, 
        ytick={0, 0.2, 0.4, 0.6, 0.8,1},
        bar width = 10pt,
        ylabel={Normalized Iteration Time},
        ylabel near ticks,
        legend style={
            draw=none,
            at={(0.5,1.0)},
            anchor=south,
            legend columns=2,
        }
        ]

        \addplot +[ybar,   fill=gray9,draw=black, area legend]       table [x={idx},  y={iters}]     {ilt.dat};
        \addplot +[ ybar,   fill=mypink,draw=black, area legend]         table [x={idx},  y={iters}]     {reuse.dat};
                \legend{
                No Pattern Matching,  
                With Pattern Matching,
                }
    \end{axis}
\end{tikzpicture} 
    \caption{Mask convergence speed comparison with/without Pattern Matching. The time of mask optimization without pattern matching is normalized to 1 to demonstrate the acceleration ratio clearly. Optimization with pattern matching takes calibrated mask as the initial state. }
    \label{exp:iterations}
\end{figure*}

\begin{table*}[tb!]
    \centering
    \caption{Comparisons of baseline approaches}
    \renewcommand{\arraystretch}{0.9}
    \resizebox{0.68\linewidth}{!}
    {
    \begin{tabular}{c|ccc|ccc|ccc}
        \toprule
         Test Case & \multicolumn{3}{c}{DAMO-DGS \cite{chen2021damo}} & \multicolumn{3}{c}{ILT-GPU \cite{gao2014mosaic}} & \multicolumn{3}{c}{AdaOPC} \\ 
         ID&\#EPE &PVB ($nm^{2}$) & RT (s)&\#EPE &PVB ($nm^{2}$) & RT (s)&\#EPE &PVB ($nm^{2}$) & RT (s)\\ 
        \midrule
         1 &22 &23323 & 5.20&23 &23329 &41.15 &22 &23232 &5.50\\
         2 &26 &26729 & 5.26&25 &26762 &48.5 &24 &26580 &5.41\\
         3 &27 &26938 & 5.22&24 &26720 &55.92 &24 &26718 &5.37\\
         4 &36 &27975 & 5.18&29 &28127 &70.57 &25 &27934 &5.40\\
         5 &35 &28805 & 5.32&30 &28925 &66.89 &30 &28927 &5.44\\
         6 &30 &26960 & 5.31&25 &26762 &55.81 &24 &26775 &5.38\\
         7 &33 &26382 & 5.23&28 &26453 &59.47 &28 &26281 &5.43\\
         8 &32 &30646 & 5.38&25 &29450 &54.88 &27 &29341 &5.42\\
         9 &25 &24054 & 5.25&24 &24053 &70.62 &23 &24022 &5.43\\
         10 &24 &21939 & 5.29&23 &21701 &37.59 &22 &21644 &5.53\\
         \midrule
         Avg. &29.0 &26375 &\textbf{5.26} &25.6 &26228 &56.14 &\textbf{24.9} &\textbf{26145} &5.43\\ 
         Ratio &1.165 &1.009 &\textbf{0.970} &1.028 &1.003 &10.340 &\textbf{1.000} &\textbf{1.000} &1.000 \\
        \bottomrule
    \end{tabular}
    }
    \label{tab:main result}
\end{table*}

\subsection{Performance Analysis}
In the begining, we validate the effectiveness of mask reuse by observing EPE descending trend. Firstly we present a demo experiment on a pattern, recording the EPE descending trend of the mask during ILT iterations with a calibrated optimized mask as the initial state. In comparison, we also record the trend with no initial mask and start the ILT process from scratch. As \Cref{fig:epe1} shows, EPE number with an initial mask starts at nearly optimal 23, and the descending trend converged at 1st iteration. In contrast, ILT from scratch gives 37 EPEs initially and takes six iterations to reach the initial EPE number with mask reuse, and overall 12 iterations to converge to 22. 

To verify the efficiency of our framework, we experiment with the runtime analysis. Since non-critical patterns are handled by extremely fast machine learning-based methods, we mainly focus on critical patterns. \Cref{fig:breakdown} visually demonstrates the time cost proportion of each step for critical pattern OPC in AdaOPC. We can see that 91.7\% running time is spent on lithography and ILT OPC iterations while the time overhead of pattern matching and shift calibration combined is only 8.3\%, which is trivial to the whole process. This also verifies the extensibility of our framework, leaving room for new powerful and fast OPC tools or litho-model to be imported for further speed up. Moreover, we also considered the condition when the pattern library gets larger. Although we cannot generate too many "ground-truth" masks due to time and computation resource limits, we can test the pattern matching speed with a large number of synthesized pattern vectors. \Cref{table:matching_time} has shown the matching speed of different pattern library sizes and embedding vector dimension combinations, from which we can see that even with the pattern library enlarged to 10000 stored patterns with dimension 512, the query and matching can still be finished in 0.4s. The time overhead is fast enough to be ignored in the whole process. \Cref{exp:iterations} shows the time comparison for mask convergence speed with or without pattern, 10 cases were tested after 800 patterns were inserted into the pattern library. Iterations of mask updates required can be significantly reduced by 93.6\% on average. 

At last, we need to verify the overall performance of our AdaOPC framework by comparing it with two baseline methods that we used in different branches for critical/non-critical patterns. After inserting 800 patterns with optimized masks into the library, we randomly tested 10 patterns, and \Cref{tab:main result} shows that AdaOPC can achieve comparable EPE/PVBand performance with ILT-GPU approach with $10\times$ acceleration and no accuracy loss. On the other hand, when compared with DAMO-DMG with 1 round lithography verification, AdaOPC shows much better performance with comparable speed. Notice that in our case, the embedding vector dimension is 256. As shown in \Cref{table:matching_time}, even if the library gets large to 10000 patterns, the matching time is still around 0.2s, which is marginal to the complete OPC process time cost shown in \Cref{tab:main result}. To summarize, the AdaOPC framework achieves performance-speed dual optimization.
\section{Conclusion}
In this paper, we proposed a self-adaptive OPC framework for mask optimization for real designs by inspecting the characteristics of a full design. We proposed an extensible OPC solver selector to choose an appropriate solver for patterns with different complexity. Additionally, we also built a dynamic pattern library to reuse optimized masks for repeating patterns with the same geometric shape. We use supervised contrastive learning to embed patterns into vectors and propose a graph-based search strategy for fast pattern matching. At last, we validate the mask reusability by proving pattern shift equivariance property and proposed a practical shift calibration tool. Extensive experiments have shown our frame can achieve OPC speed-robustness co-optimization for real design patterns. 
\section*{Acknowledgments}
This work is supported The Research Grants Council of Hong Kong SAR (Project No.~CUHK14208021).

\clearpage
{
    \bibliographystyle{IEEEtran}
    \bibliography{ref/opc} 
}

\end{document}